\documentclass[12pt]{article}

\usepackage[margin=1in]{geometry}
\usepackage{amsthm}
\usepackage{amssymb}
\usepackage{amsmath}
\usepackage{graphicx}
\usepackage{mathdots}
\usepackage{mathtools}
\usepackage{url}
\usepackage{color}
\usepackage{framed}

\newtheorem{theorem}{Theorem}%[section]

\newtheorem{lemma}[theorem]{Lemma}

\theoremstyle{definition}

\title{Neural collapse with unconstrained features}

\author{Dustin~G.~Mixon, Hans~Parshall, Jianzong~Pi}

\author{Dustin~G.~Mixon\footnote{Department of Mathematics, The Ohio State University, Columbus, OH} \footnote{Translational Data Analytics Institute, The Ohio State University, Columbus, OH} \and Hans~Parshall\footnote{Department of Mathematics, Western Washington University, Bellingham, WA} \and Jianzong~Pi\footnote{Department of Electrical and Computer Engineering, The Ohio State University, Columbus, OH}}
\date{}

\begin{document}
\maketitle

\begin{abstract}
Neural collapse is an emergent phenomenon in deep learning that was recently discovered by Papyan, Han and Donoho.
We propose a simple \textit{unconstrained features model} in which neural collapse also emerges empirically.
By studying this model, we provide some explanation for the emergence of neural collapse in terms of the landscape of empirical risk.
\end{abstract}

\section{Introduction}

Consider the task of learning a function $\mathcal{X}\to[C]$, where $\mathcal{X}\subseteq\mathbb{R}^d$ represents a space of signals and $[C]:=\{1,\ldots,C\}$ represents a set of $C\in\mathbb{N}$ labels.
Given a labeled training set $\{(x_i,c_i)\}_{i\in\mathcal{I}}$ in $\mathcal{X}\times [C]$, one might seek both a feature map $h\colon \mathcal{X} \to\mathbb{R}^p$ and a linear classifier $z\mapsto \arg\max(Wz+b)$ with $W\in\mathbb{R}^{C\times p}$ and $b\in\mathbb{R}^C$ for which empirical risk is small:
\[
R_e(h,W,b)
:=\sum_{i\in\mathcal{I}} \mathcal{L}(Wh(x_i)+b,y_i).
\]
Here, $\mathcal{L}\colon\mathbb{R}^C\times \mathbb{R}^C\to\mathbb{R}^+$ denotes some loss function and $y_i\in\mathbb{R}^C$ denotes the one-hot vector representation $e_{c_i}$ of the label $c_i$.
To accomplish this, it is common to take $h$ to reside in some parameterized family $\mathcal{H}=\{h_\theta:\theta\in\Theta\}$, such as a family of neural networks, and then locally minimize empirical risk over $(\theta,W,b)$.
Such methods have revolutionized classification in various domains, notably image domains~\cite{KrizhevskySH:12}, but this level of performance is largely under-explained by existing theory.

Papyan, Han and Donoho~\cite{PapyanHD} recently observed several emergent phenomena in the \textit{terminal phase of training}, that is, when the above empirical risk is minimized even after the classifier $x\mapsto \arg\max(Wh(x)+b)$ interpolates the training set.
They focused on the balanced case in which there exists $N\in\mathbb{N}$ such that for each $c\in[C]$, it holds that
\[
|\{i\in\mathcal{I}:c_i=c\}|=N,
\]
i.e., $|\mathcal{I}|=CN$.
For several common families $\mathcal{H}$ of neural networks, locally minimizing empirical risk results in $(h,W,b)$ that satisfies several properties, which are collectively referred to as \textbf{neural collapse}:

\medskip

\noindent
\textbf{(NC1) Variability collapse.}
For each $c\in [C]$, there exists $\mu_c\in\mathbb{R}^p$ such that for every $i\in\mathcal{I}$ with $y_i=e_c$, it holds that $h(x_i)=\mu_c$.

\medskip

\noindent
\textbf{(NC2) Simplex equiangular tight frame structure.}
Put $\mu_G:=\frac{1}{C}\sum_{c=1}^C\mu_{c}$.
Then
\[
\|\mu_c-\mu_G\|_2
=\|\mu_{c'}-\mu_G\|_2
\]
for every $c,c'$.
Furthermore, let the $c$th column of $M\in\mathbb{R}^{p\times C}$ be
$\tilde\mu_c
:=\frac{\mu_c-\mu_G}{\|\mu_c-\mu_G\|_2}$.
Then
\[
M^\top M
=\frac{C}{C-1}I_C-\frac{1}{C-1}1_C1_C^\top.
\]

\medskip

\noindent
\textbf{(NC3) Self-duality.}
$\frac{W^\top}{\|W\|_F}=\frac{M}{\|M\|_F}$.

\medskip

\noindent
\textbf{(NC4) Equivalence to nearest class center.}
$\arg\max_c(Wz+b)_c=\arg\min_c\|z-\mu_c\|_2$.

\medskip

In general, an equiangular tight frame (ETF) is any tuple $\{v_i\}_{i=1}^n$ of unit vectors in $\mathbb{C}^d$ for which there exist $\alpha,\beta\geq0$ such that
\[
\sum_{i=1}^n v_iv_i^*=\alpha I_d,
\qquad
|\langle v_i,v_j\rangle|^2=\beta
\qquad
\forall i,j\in[n],~i\neq j.
\]
It is natural to identify each $v_iv_i^*$ with the corresponding point in complex projective space.
ETFs were introduced in~\cite{StrohmerH:03} as convenient maximizers of minimum pairwise distance in this space.
As optimal projective codes, ETFs find applications in multiple description coding~\cite{StrohmerH:03}, digital fingerprinting~\cite{MixonQKM:13}, compressed sensing~\cite{BandeiraFMW:13},
and quantum state tomography~\cite{RenesBSC:04}.
These applications have motivated a flurry of recent work to discover various infinite families of ETFs; see~\cite{FickusM:15} for a living survey.
The simplex ETF arises from one of the simplest constructions: take $\{v_i\}_{i=1}^n$ to be the $n=d+1$ vertices of an origin-centered regular simplex in $\mathbb{R}^d$.
We note that negating any vector in an ETF produces another ETF (as ETFs are fundamentally a projective objects), whereas property (NC2) is not invariant to such an operation.
The regular simplex also emerges as an optimal spherical code by virtue of achieving equality in Rankin's simplex bound~\cite{Rankin:55}.
In addition, $C$ vertices of a regular simplex in $\mathbb{R}^p$ induce Voronoi cells that partition $\mathbb{R}^p$ into $C$ isometric cones that are highly symmetric, and this geometry plays a fundamental role in (NC4).

The rich geometric structure of neural collapse has some advantages in the context of machine learning. 
In particular, if the feature map $h$ generalizes, then neural collapse helps the classifier $x\mapsto\arg\max(Wh(x)+b)$ to also generalize.
Indeed, consider a member $x$ of the test set corresponding to class $c$.
If the corresponding point $h(x)$ in the feature domain is a perturbation of $\mu_c$, then $h(x)$ is most likely to reside in the Voronoi cell containing $\mu_c$ if the class means form the vertices of a regular simplex.
Papyan, Han and Donoho~\cite{PapyanHD} provide some theoretical justification along these lines.

While the notion of neural collapse appears useful in both theory and practice, it remains unexplained why neural collapse emerges from empirical risk minimization.
This paper provides some explanation for this emergence.
In the next section, we propose a simple \textit{unconstrained features model} in which a strong notion of neural collapse empirically emerges.
For this model, Section~3 identifies an invariant subspace of the gradient descent dynamical system that encourages convergence to this strong notion of neural collapse.
We conclude in Section~4 with a brief discussion.

\section{Unconstrained features and strong neural collapse}

Empirically, neural collapse is a phenomenon that occurs in the terminal phase of training, in which the trained classifier interpolates the training set.
This behavior is only feasible when the feature maps in $\mathcal{H}$ restricted to the training set $\{x_i\}_{i\in\mathcal{I}}$ form a high-dimensional subset of $(\mathbb{R}^p)^{\mathcal{I}}$.
Indeed, the observations by Papyan, Han and Donoho~\cite{PapyanHD} were made for over-parameterized families of neural networks.
In this paper, we take this to an extreme of sorts by considering the following \textbf{unconstrained features model}: 
\[
\mathcal{H}
:=\{h\colon\mathcal{X}\to\mathbb{R}^p\},
\qquad
\mathcal{X}:=\{x_i:i\in\mathcal{I}\}\subseteq\mathbb{R}^d.
\]
Under the reasonable assumption that $x_i\neq x_j$ whenever $i\neq j$, then we may represent $\mathcal{H}$ by the matrix space $\mathbb{R}^{p\times CN}$; explicitly, we identify $\mathcal{I}$ with $[CN]$, and the $i$th column of $H\in\mathbb{R}^{p\times CN}$ is given by $h(x_i)$.
Recalling that $\{c_i\}_{i\in\mathcal{I}}$ consists of $N$ copies of each member of $[C]$, we may similarly represent $\{y_i\}_{i\in\mathcal{I}}$ as the columns of the matrix $I_C\otimes 1_N^\top$.
Note that feature maps in the unconstrained features model are only defined over the training set, while feature maps that are typically trained in practice are defined over all of $\mathbb{R}^d$.
Of course, such a choice of $\mathcal{H}$ has no hope of generalizing to a test set, but we will find that it facilitates the study of neural collapse.
For simplicity, we consider the loss defined by $\mathcal{L}(u,v):=\frac{1}{2}\|u-v\|_2^2$, in which case empirical risk reduces to
\[
R_e(H,W,b)
=\frac{1}{2}\|WH+b1_{CN}^\top-I_C\otimes1_N^\top \|_F^2,
\qquad
H\in\mathbb{R}^{p\times CN},~W\in\mathbb{R}^{C\times p},~b\in\mathbb{R}^C.
\]
Following common practice in deep learning~\cite{UFLDL:online}, we initialize gradient descent close to the origin.
As illustrated in Figure~\ref{fig.snc}, the following phenomena emerges from this local optimization, which we collectively refer to as \textbf{strong neural collapse}:
\begin{align}
\label{eq.snc1}
\tag{SNC1}
WW^\top
&=\sqrt{N}(I_C-\frac{1}{C}1_C1_C^\top),\\
\label{eq.snc2}
\tag{SNC2}
H
&=\frac{1}{\sqrt{N}}(W\otimes 1_N)^\top,\\
\label{eq.snc3}
\tag{SNC3}
b
&=\frac{1}{C}1_C.
\end{align}
The following lemma establishes that the points exhibiting strong neural collapse form a subset of the global minimizers of empirical risk.
In the next section, we show how the emergence of strong neural collapse is an artifact of the optimization landscape.

\begin{figure}
\begin{center}
\includegraphics[height=0.18\textwidth,trim={20 20 20 20},clip]{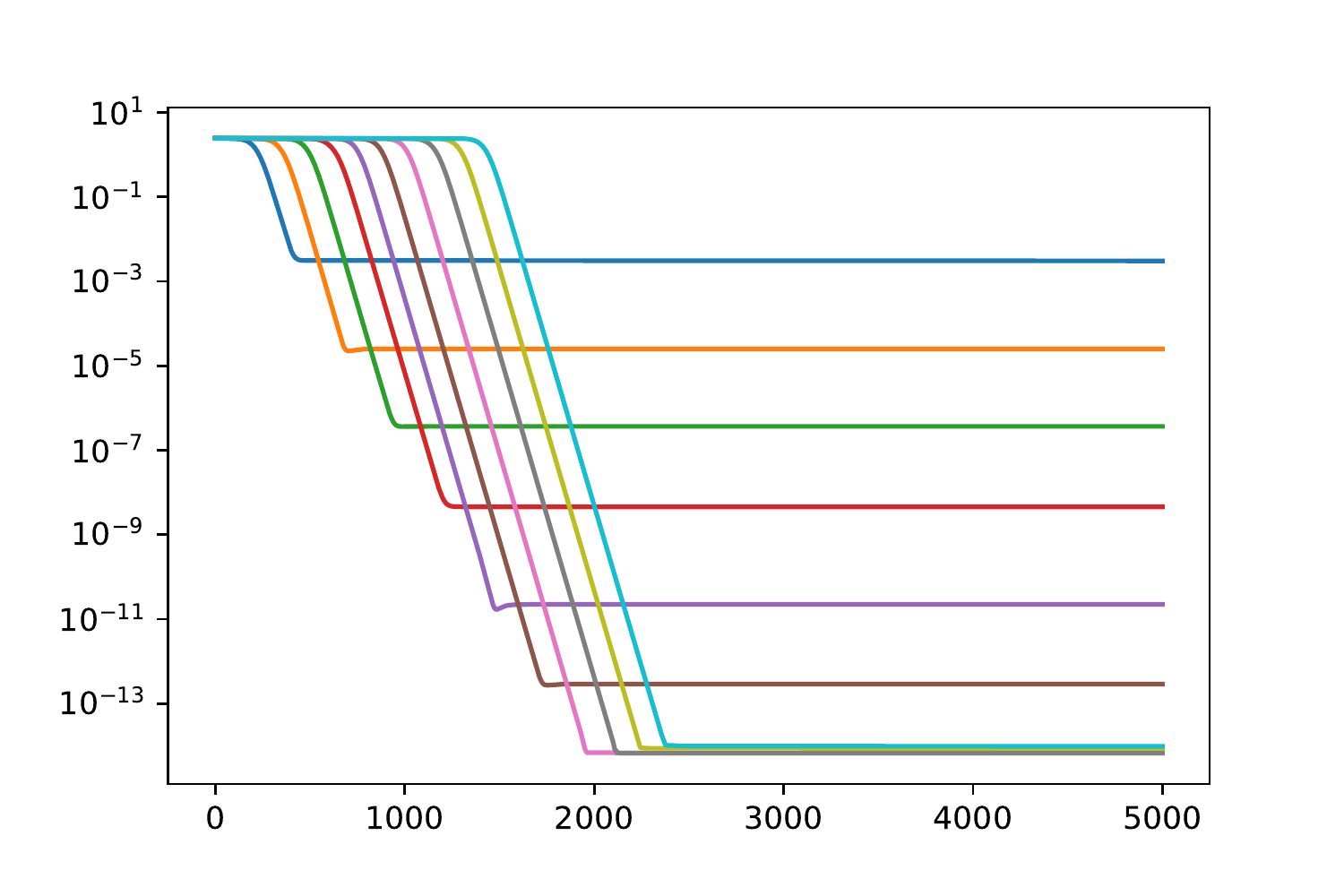}
\includegraphics[height=0.18\textwidth,trim={20 20 20 20},clip]{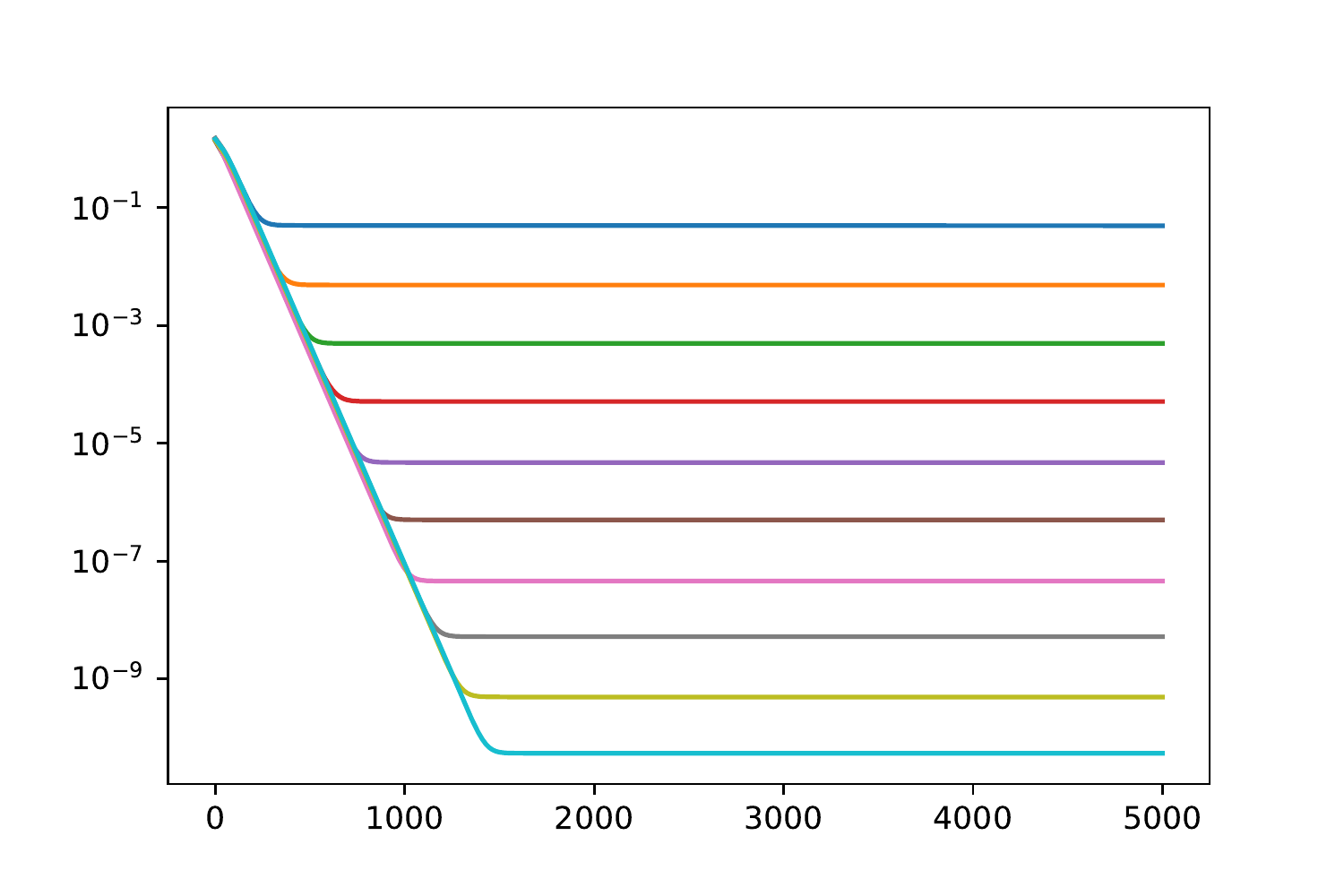}
\includegraphics[height=0.18\textwidth,trim={20 20 20 20},clip]{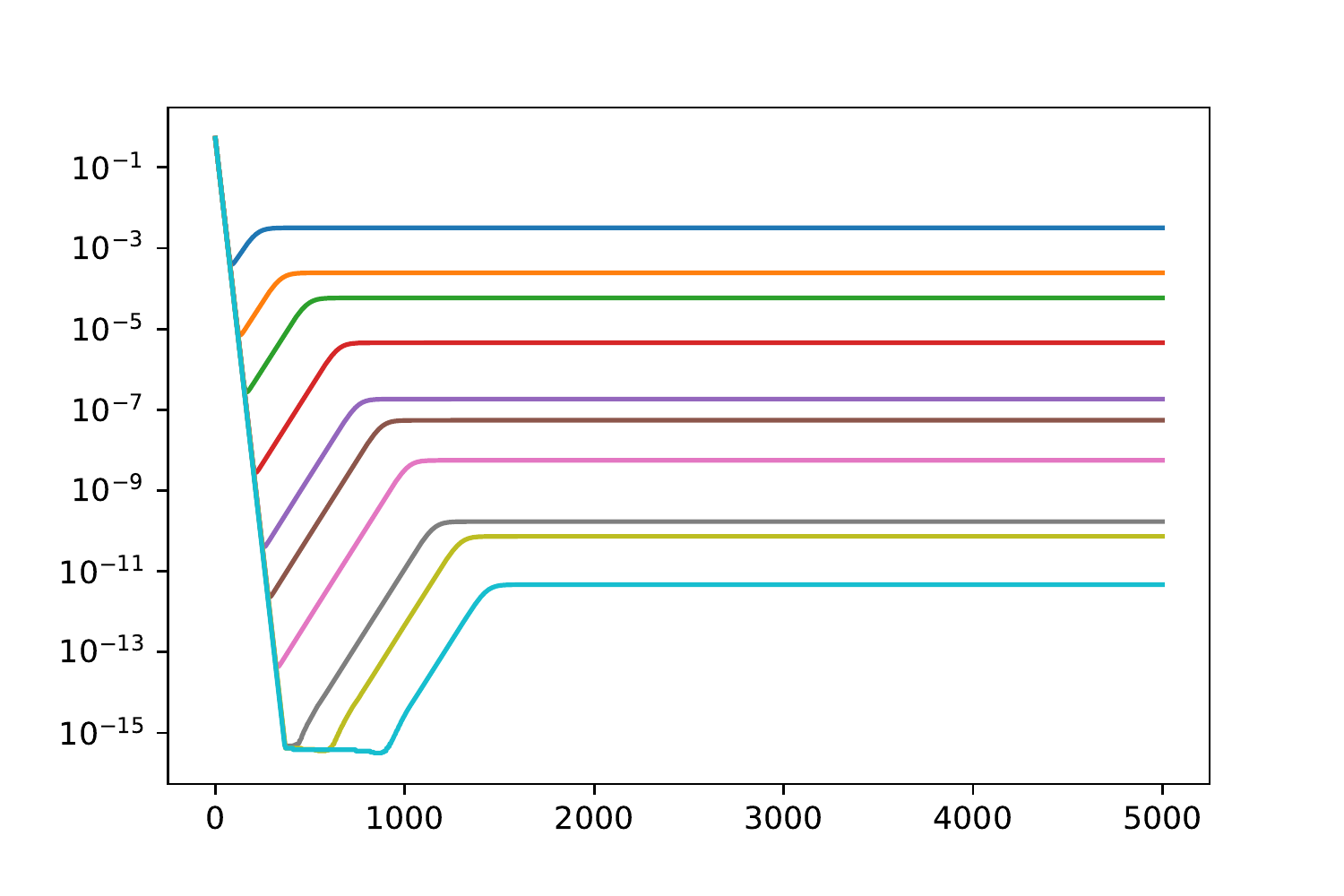}
\includegraphics[height=0.18\textwidth]{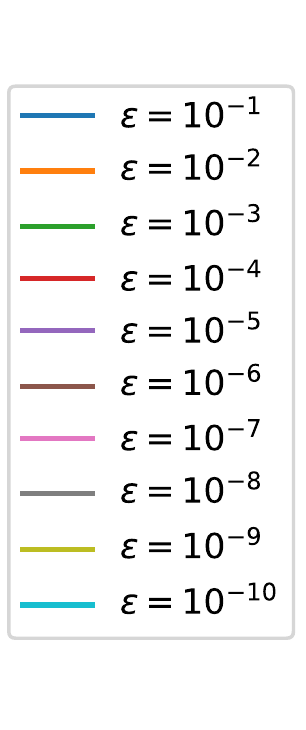}
\end{center}
\caption{\label{fig.snc}
The emergence of strong neural collapse.
Run gradient decent to minimize $R_e(H,W,b)$ for $C=N=3$ and $p=15$, initializing at a random choice of $H_0$ and $W_0$ with $\|H_0\|_F=\|W_0\|_F=\varepsilon$ and $b_0=0$.
At each iteration, quantify the error in \eqref{eq.snc1} by $\|WW^\top-\sqrt{N}(I_C-\frac{1}{C}1_C1_C^\top)\|_F$ (plotted on the left), the relative error in \eqref{eq.snc2} by $\|H-\frac{1}{\sqrt{N}}(W\otimes 1_N)^\top\|_F/\|H\|_F$ (plotted in the middle), and the error in \eqref{eq.snc3} by $\|b-\frac{1}{C}1_C\|_2$ (plotted on the right).
Apparently, the limit point of gradient descent approaches strong neural collapse as the initialization approaches the origin.
}
\end{figure}

\begin{lemma}\
\begin{itemize}
\item[(a)]
If $(H,W,b)$ exhibits strong neural collapse, then $R_e(H,W,b)=0$.
\item[(b)]
Strong neural collapse implies neural collapse.
\end{itemize}
\end{lemma}

\begin{proof}
For (a), we have
\begin{align*}
WH+b1_{CN}^\top-I_C\otimes 1_N^\top
&=W(\frac{1}{\sqrt{N}}(W\otimes 1_N)^\top)+\frac{1}{C}1_C1_{CN}^\top-I_C\otimes 1_N^\top\\
&=(\frac{1}{\sqrt{N}} WW^\top + \frac{1}{C}1_C1_C^\top - I_C)\otimes 1_N^\top,
\end{align*}
which vanishes by \eqref{eq.snc1}.
Next, we consider (b).
Letting $w_c^\top$ denote the $c$th row of $W$, then (NC1) follows from \eqref{eq.snc2} by taking $\mu_c:=\frac{1}{\sqrt{N}}w_c$.
For (NC2), observe that \eqref{eq.snc1} implies $\|W^\top 1_C\|_2^2=1_C^\top W W^\top 1_C=0$, and so
\[
\mu_G
=\frac{1}{C}\sum_{c=1}^C\mu_c
=\frac{1}{C}\sum_{c=1}^C\frac{1}{\sqrt{N}}w_c
=\frac{1}{C\sqrt{N}}W^\top1_C
=0.
\]
Then for every $c\in[C]$, it holds that
\[
\|\mu_c-\mu_G\|_2^2
=\|\mu_c\|_2^2
=\|\frac{1}{\sqrt{N}}w_c\|_2^2
=\frac{1}{N}(WW^\top)_{cc}
=\frac{1}{\sqrt{N}}(1-\frac{1}{C}).
\]
Furthermore, $\tilde\mu_c=\frac{\mu_c}{\|\mu_c\|_2}=\frac{w_c}{\|w_c\|_2}$, and so
\[
(M^\top M)_{cc'}
=\langle \tilde\mu_c,\tilde\mu_{c'}\rangle
=\langle \frac{w_c}{\|w_c\|_2},\frac{w_{c'}}{\|w_{c'}\|_2}\rangle
=\frac{(WW^\top)_{cc'}}{\sqrt{N}(1-\frac{1}{C})}
=(\frac{C}{C-1}I_C-\frac{1}{C-1}1_C1_C^\top)_{cc'}.
\]
Since $\|w_c\|_2^2=\sqrt{N}(1-\frac{1}{C})$ is constant over $c\in [C]$, we also have $M\propto W^\top$, from which (NC3) follows.
It remains to verify (NC4).
For this, we first apply \eqref{eq.snc3} and the fact that $\mu_c=\frac{1}{\sqrt{N}}w_c$ to obtain
\[
\arg\max_c(Wz+b)_c
=\arg\max_c(Wz)_c
=\arg\max_c\langle z,w_c\rangle
=\arg\max_c\langle z,\mu_c\rangle.
\]
Finally, \eqref{eq.snc1} implies $\|\mu_c\|_2=\|\mu_{c'}\|_2$ for every $c,c'$, and so
\[
\arg\max_c\langle z,\mu_c\rangle
=\arg\min_c(\|z\|_2^2-2\langle z,\mu_c\rangle+\|\mu_c\|_2^2)
=\arg\min_c\|z-\mu_c\|_2^2.
\qedhere
\]
\end{proof}

\section{The effect of empirical risk minimization}

Write $Z:=(H,W,b)$ and consider the \textit{gradient flow} ordinary differential equation
\[
Z'(t)=-\nabla R_e(Z(t)).
\]
This serves as a model for gradient descent with small learning rate that has been used to analyze the training of neural networks~\cite{AroraCH:18,ChizatB:18,DuHL:18,DuZPS:18,SongMN:18}.
To analyze gradient flow, we start by computing the gradient:
\[
\nabla_H R_e=W^\top A,
\quad
\nabla_W R_e=AH^\top,
\quad
\nabla_b R_e=A1_{CN},
\quad
A:=WH+b1_{CN}^\top-I_C\otimes 1_N^\top.
\]
Notably, the resulting gradient flow differential equation is nonlinear.
In order to analyze the initial behavior of our trajectory, we consider a modification in which $A$ is replaced by $\tilde{A}:=b1_{CN}^\top-I_C\otimes 1_N^\top$.
By partially decoupling $(H,W)$ and $b$, this modification facilitates analysis, while serving as an good approximation in the regime where $(H,W)$ is small.

\begin{theorem}
\label{thm.initial growth}
The solution to the ordinary differential equation
\begin{align*}
H'(t)
&=-W(t)^\top \tilde{A}(t),
\qquad
W'(t)
=-\tilde{A}(t)H(t)^\top,\\
b'(t)
&=-\tilde{A}(t)1_{CN},
\qquad
\tilde{A}(t)
:=b(t)1_{CN}^\top-I_C\otimes 1_N^\top
\end{align*}
with initial condition $H(0)=H_0$, $W(0)=W_0$, $b(0)=0$ satisfies
\[
\|(H(t),W(t))-e^{\sqrt{N}t}\cdot\Pi_T(H_0,W_0)\|_E\leq e^{1/(C\sqrt{N})}\cdot\|\Pi_{T^\perp}(H_0,W_0)\|_E,
\quad
b(t)=(\frac{1-e^{-CNt}}{C})1_C
\]
for all $t\geq0$, where $\|(H,W)\|_E^2:=\|H\|_F^2+\|W\|_F^2$ and $\Pi_T$ denotes orthogonal projection onto the subspace
\[
T:=\{(H,W):H=\frac{1}{\sqrt{N}}(W\otimes1_N)^\top,~1_C^\top W=0\}.
\]
\end{theorem}

\begin{proof}
We start by solving for $b(\cdot)$, which is governed by the differential equation
\[
b'(t)
=-\tilde{A}(t)1_{CN}
=(I_C\otimes 1_N^\top-b(t)1_{CN}^\top)1_{CN}
=N(1_C-Cb(t)).
\]
Since $b(0)=0$, we may write $b(t)=\beta(t)1_C$, in which case $\beta'(t)=N(1-C\beta(t))$, and so $\beta(t)=\frac{1-e^{-CNt}}{C}$, as claimed.
Next, writing $U=(H,W)$, then the unsolved portion of our system is given by
\[
U'(t)=L_t(U(t)),
\quad
L_t(H,W):=(W^\top M_t, M_tH^\top),
\quad
M_t:=(I_C\otimes 1_N^\top)-\beta(t)1_C1_{CN}^\top.
\]
First, we observe that each $L_t$ is self-adjoint:
\begin{align*}
\langle L_t(H,W),(\tilde{H},\tilde{W})\rangle_E
&=\langle (W^\top M_t, M_tH^\top),(\tilde{H},\tilde{W})\rangle_E\\
&=\operatorname{tr}((W^\top M_t)^\top\tilde{H})+\operatorname{tr}((M_tH^\top)^\top\tilde{W})\\
&=\operatorname{tr}(M_t^\top W \tilde{H})+\operatorname{tr}(HM_t^\top\tilde{W})\\
&=\operatorname{tr}(\tilde{H}M_t^\top W)+\operatorname{tr}(M_t^\top\tilde{W}H)\\
&=\operatorname{tr}((M_t\tilde{H}^\top)^\top W)+\operatorname{tr}((\tilde{W}^\top M_t)^\top H)\\
&=\langle (\tilde{W}^\top M_t,M_t\tilde{H}^\top),(H,W)\rangle_E
=\langle (H,W),L_t(\tilde{H},\tilde{W})\rangle_E.
\end{align*}
Next, we claim that $\{L_t\}_{t\geq0}$ are simultaneously diagonalizable over five eigenspaces:
\begin{align*}
E_1^\epsilon
&:=\{(H,W):H=\epsilon\cdot\frac{1}{\sqrt{N}}(W\otimes1_N)^\top,~1_C^\top W=0\},\\
E_2^\epsilon
&:=\{(H,W):H=\epsilon\cdot z1_{CN}^\top,~W=\sqrt{N}1_Cz^\top,~z\in\mathbb{R}^p\},\\
E_3
&:=\{(H,W):(I_C\otimes 1_N^\top)H^\top=0,~W=0\},
\end{align*}
where $\epsilon\in\{\pm\}$.
To see this, first note that
\[
\operatorname{dim}E_1^\epsilon
=p(C-1),
\qquad
\operatorname{dim}E_2^\epsilon
=p,
\qquad
\operatorname{dim}E_3
=pC(N-1).
\]
Since these dimensions sum to $pC(N+1)=\operatorname{dim}(\mathbb{R}^{p\times CN}\oplus\mathbb{R}^{C\times p})$, it suffices to show that each nonzero member of each claimed eigenspace is an eigenvector, and that the claimed eigenspaces have distinct eigenvalues.
First, suppose $(H,W)\in E_1^\epsilon$.
Then
\begin{align*}
W^\top M_t
&=W^\top((I_C\otimes 1_N^\top)-\beta(t)1_C1_{CN}^\top)
=W^\top(I_C\otimes 1_N^\top)
=W^\top\otimes 1_N^\top
=\epsilon\sqrt{N}\cdot H,\\
M_tH^\top
&=((I_C\otimes 1_N^\top)-\beta(t)1_C1_{CN}^\top)(\epsilon\cdot\frac{1}{\sqrt{N}}(W\otimes1_N)^\top)^\top\\
&=\epsilon\cdot\frac{1}{\sqrt{N}}\cdot((I_C\otimes 1_N^\top)-\beta(t)1_C1_C^\top\otimes 1_N^\top)(W\otimes1_N)=\epsilon\sqrt{N}\cdot W,
\end{align*}
i.e., $(H,W)$ is an eigenvector of $L_t$ with eigenvalue $\epsilon\sqrt{N}$.
Next, suppose $(H,W)\in E_2^\epsilon$.
Then
\begin{align*}
W^\top M_t
&=(\sqrt{N}1_Cz^\top)^\top((I_C\otimes 1_N^\top)-\beta(t)1_C1_{CN}^\top)\\
&=\sqrt{N}z1_C^\top((I_C\otimes 1_N^\top)-\beta(t)1_C1_{CN}^\top)\\
&=\sqrt{N}z(1_C^\top\otimes 1_N^\top-C\beta(t)1_{CN}^\top)
=\sqrt{N}(1-C\beta(t))\cdot z1_{CN}^\top
=\epsilon\sqrt{N}(1-C\beta(t))\cdot H,\\
M_tH^\top
&=((I_C\otimes 1_N^\top)-\beta(t)1_C1_{CN}^\top)(\epsilon\cdot z1_{CN}^\top)^\top\\
&=\epsilon\cdot((I_C\otimes 1_N^\top)-\beta(t)1_C1_{CN}^\top)1_{CN}z^\top\\
&=\epsilon\cdot((I_C\otimes 1_N^\top)(1_C\otimes1_N)-\beta(t)1_C1_{CN}^\top1_{CN})z^\top\\
&=\epsilon\cdot(N1_C-CN\beta(t)1_C)z^\top
=\epsilon N(1-C\beta(t))\cdot 1_Cz^\top
=\epsilon\sqrt{N}(1-C\beta(t))\cdot W,
\end{align*}
i.e., $(H,W)$ is an eigenvector of $L_t$ with eigenvalue $\epsilon\sqrt{N}(1-C\beta(t))$.
Finally, suppose $(H,W)\in E_3$.
Then $W^\top M_t=0$ and
\[
M_tH^\top
=((I_C\otimes 1_N^\top)-\beta(t)1_C1_{CN}^\top)H^\top
=-\beta(t)1_C1_{CN}^\top H^\top
=-\beta(t) 1_C1_C^\top(I_C\otimes 1_N^\top)H^\top
=0,
\]
i.e., $(H,W)$ is an eigenvector of $L_t$ with eigenvalue $0$.
Overall, letting $\Pi_i^\epsilon$ denote orthogonal projection onto $E_i^\epsilon$, we have the spectral decomposition
\[
L_t
=\sqrt{N}\Big(\Pi_1^+-\Pi_1^-+(1-C\beta(t))\Pi_2^+-(1-C\beta(t))\Pi_2^-\Big).
\]
Finally, we solve the differential equation $U'(t)=L_t(U(t))$ by finding the orthogonal projection of $U(t)$ onto each eigenspace of $L_t$.
First, $\Pi_1^\epsilon U'(t)=\epsilon\sqrt{N}\Pi_1^\epsilon U(t)$, and so
\[
\Pi_1^\epsilon U(t)
=e^{\epsilon\sqrt{N}t} \Pi_1^\epsilon U(0).
\]
Next, $\Pi_2^\epsilon U'(t)=\epsilon\sqrt{N}(1-C\beta(t))\Pi_2^\epsilon U(t)$, and so
\[
\Pi_2^\epsilon U(t)
=f_\epsilon(t) \Pi_2^\epsilon U(0),
\qquad
f_\epsilon(t)
:=\operatorname{exp}(\frac{\epsilon}{C\sqrt{N}}(1-e^{-CNt})).
\]
An application of the Pythagorean theorem then gives
\begin{align*}
\|U(t)-e^{\sqrt{N}t}\Pi_1^+ U(0)\|_E^2
&=\|e^{-\sqrt{N}t}\Pi_1^-U(0)+f_+(t)\Pi_2^+U(0)+f_-(t)\Pi_2^-U(0)\|_E^2\\
&=e^{-2\sqrt{N}t}\|\Pi_1^-U(0)\|_E^2+f_+(t)^2\|\Pi_2^+U(0)\|_E^2+f_-(t)^2\|\Pi_2^-U(0)\|_E^2\\
&\leq \|\Pi_1^-U(0)\|_E^2+e^{2/(C\sqrt{N})}\|\Pi_2^+U(0)\|_E^2+\|\Pi_2^-U(0)\|_E^2\\
&\leq e^{2/(C\sqrt{N})}\|(I-\Pi_1^+)U(0)\|_E^2.
\end{align*}
The result then follows by observing that $T=E_1^+$.
\end{proof}

Theorem~\ref{thm.initial growth} indicates that gradient flow initially magnifies the portion of $(H_0,W_0)$ that resides in $T$ while sending $b$ to $\frac{1}{C}1_C$ along $\operatorname{span}\{1_C\}$.
As a consequence, the initial trajectory of $(H,W,b)$ approximately travels along the subspace
\[
S:=\Big\{(H,W,b):H=\frac{1}{\sqrt{N}}(W\otimes 1_N)^\top,~1_C^\top W=0,~b\in\operatorname{span}\{1_C\}\Big\}.
\]
In fact, Figure~\ref{fig.invariant subspace} illustrates that (empirically) the \textit{full} trajectory approximately travels along this subspace.
As the following result demonstrates, $S$ is an invariant subspace of our differential equation that encourages convergence to points that exhibit strong neural collapse.

\begin{figure}
\begin{center}
\includegraphics[height=0.3\textwidth,trim={20 20 20 20},clip]{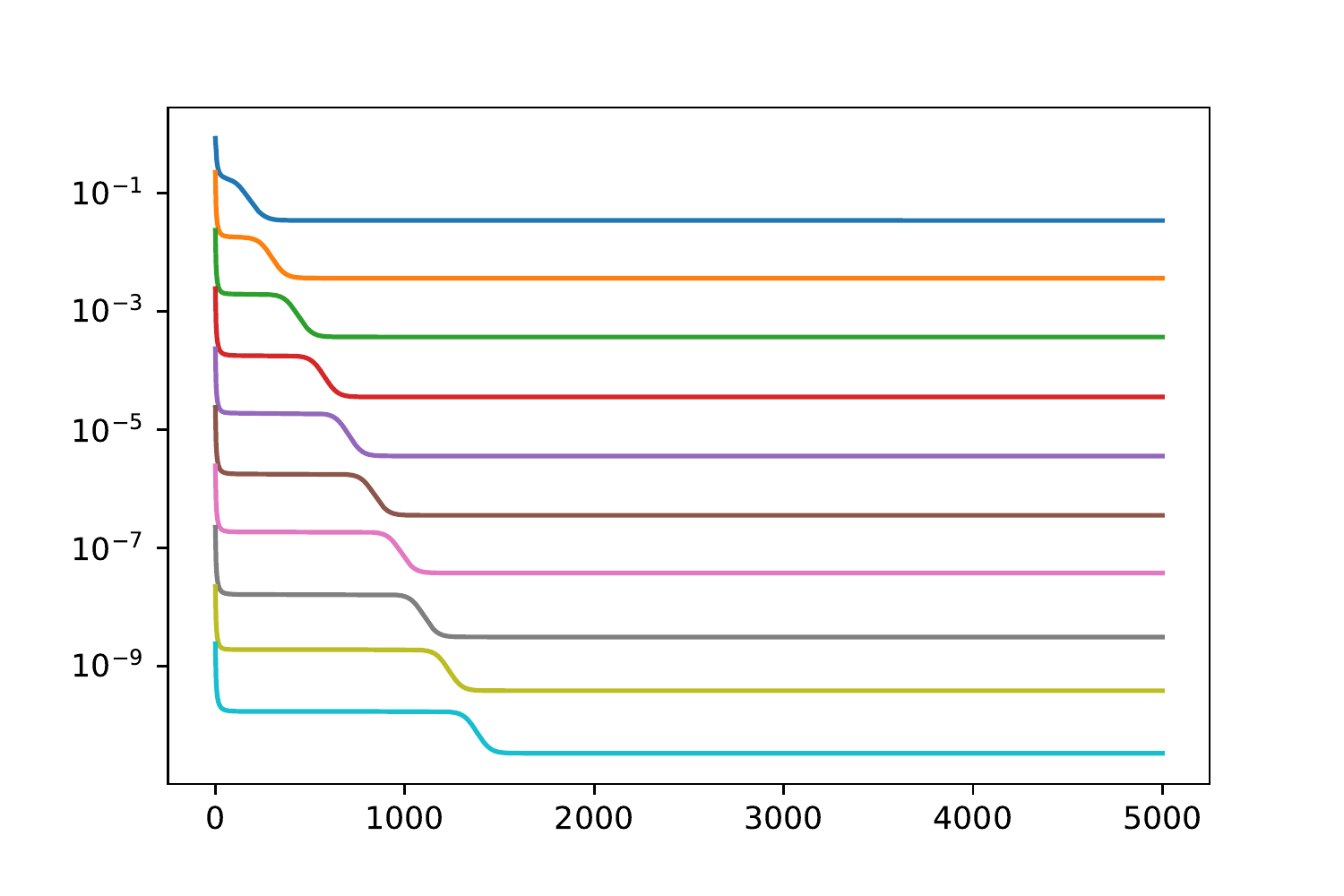}
\quad
\includegraphics[height=0.3\textwidth]{legend.pdf}
\end{center}
\caption{\label{fig.invariant subspace}
Gradient descent maintains small distance from the invariant subspace $S$.
Run gradient decent to minimize $R_e(H,W,b)$ for $C=N=3$ and $p=15$, initializing at a random choice of $H_0$ and $W_0$ with $\|H_0\|_F=\|W_0\|_F=\varepsilon$ and $b_0=0$.
At each iteration, quantify the relative distance from $S$ by $\|Z-\Pi_SZ\|_E/\|Z\|_E$, where $\|(H,W,b)\|_E^2:=\|H\|_F^2+\|W\|_F^2+\|b\|_2^2$.
}
\end{figure}

\begin{theorem}
Select $Z_0:=(H_0,W_0,b_0)\in S$ and consider $Z(\cdot)$ such that
\[
Z'(t)=-\nabla R_e(Z(t)),
\qquad
Z(0)=Z_0.
\]
\begin{itemize}
\item[(a)]
$Z(t)\in S$ for every $t\geq0$.
\item[(b)]
Select $\alpha\in\mathbb{R}$ such that $b_0=\frac{1-\alpha}{C}1_C$.
Then $(H(\cdot),W(\cdot),b(\cdot))=Z(\cdot)$ is given by
\begin{align*}
W'(t)
&=\sqrt{N}W(t)-W(t)W(t)^\top W(t),
\qquad
W(0)=W_0,\\
H(t)
&=\frac{1}{\sqrt{N}}(W(t)\otimes 1_N)^\top,
\qquad
b(t)
=\frac{1-\alpha e^{-CNt}}{C}1_C.
\end{align*}
\item[(c)]
Let $\Pi$ denote orthogonal projection onto $\operatorname{ker}W_0^\top$.
Then
\[
\lim_{t\to\infty}W(t)W(t)^\top=\sqrt{N}(I-\Pi).
\]
In particular, if $\operatorname{rank}W_0\geq C-1$, then $\lim_{t\to\infty}Z(t)$ exhibits strong neural collapse.
\end{itemize}
\end{theorem}

\begin{proof}
We start by verifying (a).
Since $S$ is a subspace, it suffices to show that $Z:=(H,W,b)\in S$ implies $\nabla R_e(Z)\in S$.
First, the constraint $1_C^\top W=0$ implies
\[
H1_{CN}
=\frac{1}{\sqrt{N}}(W\otimes 1_N)^\top(1_C\otimes 1_N)
=\frac{1}{\sqrt{N}}(W^\top1_C\otimes 1_N^\top1_N)
=0.
\]
This then implies
\begin{align}
\nabla_b R_e(Z)
\nonumber
&=(WH+b1_{CN}^\top-I_C\otimes 1_N^\top)1_{CN}\\
\label{eq.grad flow b}
&=CNb-(I_C\otimes 1_N^\top)(1_C\otimes 1_N)
=CNb-N1_C
\in\operatorname{span}\{1_C\}.
\end{align}
In addition, we have
\begin{align*}
1_C^\top \nabla_W R_e(Z)
&=1_C^\top (WH+b1_{CN}^\top-I_C\otimes 1_N^\top)H^\top\\
&=-(1_C^\top\otimes 1)(I_C\otimes 1_N^\top)H^\top
=-1_{CN}^\top H^\top
=0.
\end{align*}
It remains to verify that $\nabla_H R_e(Z)=\frac{1}{\sqrt{N}}(\nabla_W R_e(Z)\otimes 1_N)^\top$.
To this end, writing $b=\beta 1_C$ with $\beta\in\mathbb{R}$ gives $W^\top b=\beta W^\top 1_C=0$.
As such, we have
\begin{align*}
\nabla_H R_e(Z)
&=W^\top (WH+b1_{CN}^\top-I_C\otimes 1_N^\top)\\
&=W^\top WH-W^\top(I_C\otimes 1_N^\top)\\
&=W^\top W\frac{1}{\sqrt{N}}(W\otimes 1_N)^\top-W^\top(I_C\otimes 1_N^\top)
=W^\top(\frac{1}{\sqrt{N}}WW^\top-I_C)\otimes 1_N^\top.
\end{align*}
Next,
\begin{align}
\nabla_W R_e(Z)
\nonumber
&=(WH+b1_{CN}^\top-I_C\otimes 1_N^\top)H^\top\\
\nonumber
&=WHH^\top-(I_C\otimes 1_N^\top)H^\top\\
\nonumber
&=\frac{1}{N}W(W\otimes 1_N)^\top(W\otimes 1_N)-\frac{1}{\sqrt{N}}(I_C\otimes 1_N^\top)(W\otimes 1_N)\\
\label{eq.grad flow W}
&=WW^\top W-\sqrt{N}W,
\end{align}
and so $\frac{1}{\sqrt{N}}(\nabla_W R_e(Z)\otimes 1_N)^\top=\nabla_H R_e(Z)$, as desired.

For (b), we see from \eqref{eq.grad flow b} that
\[
b'(t)
=-\nabla_b R_e(Z(t))
=N(1_C-Cb(t)).
\]
Writing $b(t)=\beta(t)1_C$, then $\beta'(t)=N(1-C\beta(t))$, and so $\beta(t)=\frac{1-\alpha e^{-CNt}}{C}$, as claimed.
Also, we see from \eqref{eq.grad flow W} that
\[
W'(t)
=-\nabla_W R_e(Z(t))
=\sqrt{N}W(t)-W(t)W(t)^\top W(t).
\]
The expression for $H(t)$ follows from the fact that $Z(t)\in S$.

For (c), consider $G(t):=W(t)W(t)^\top$.
Then
\begin{align*}
G'(t)
&=W'(t)W(t)^\top+W(t)W'(t)^\top\\
&=(\sqrt{N}W(t)-W(t)W(t)^\top W(t))W(t)^\top+W(t)(\sqrt{N}W(t)-W(t)W(t)^\top W(t))^\top\\
&=2\sqrt{N}W(t)W(t)^\top-2W(t)W(t)^\top W(t)W(t)^\top\\
&=2\sqrt{N}G(t)-2G(t)^2.
\end{align*}
Since $G'(t)$ and $G(t)$ are simultaneously diagonalizable, it follows that $G(\cdot)$ takes the form $G(t)=\sum_i \lambda_i(t)\Pi_i$, where $\{\Pi_i\}$ denote orthogonal projections onto the eigenspaces of $G(0)$.
Furthermore, we have
\[
\lambda_i'(t)
=2\lambda_i(t)(\sqrt{N}-\lambda_i(t)).
\]
Since $G(0)$ is positive semidefinite, we have $\lambda_i(0)\geq0$ for every $i$.
If $\lambda_i(0)=0$, then $\lambda_i(t)=0$ for all $t\geq0$.
If $\lambda_i(0)>0$, then $\lim_{t\to\infty}\lambda_i(t)=\sqrt{N}$.
It follows that
\[
\lim_{t\to\infty}G(t)
=\lim_{t\to\infty}\sum_i \lambda_i(t)\Pi_i
=\sum_{\lambda_i(0)>0}\sqrt{N}\Pi_i
=\sqrt{N}(I-\Pi).
\]
For the last claim, we first show that $\lim_{t\to\infty}Z(t)$ exists.
Since $\lim_{t\to\infty}b(t)$ exists and $H(t)$ is determined by $W(t)$, it suffices to show that $\lim_{t\to\infty}W(t)$ exists.
Observe that the singular vectors of $W(t)$ are also singular vectors of $W'(t)$, and so $W(t)=\sum_j \sigma_j(t)u_jv_j^\top$, where $\sum_j \sigma_j(0)u_jv_j^\top$ denotes any singular value decomposition of $W_0$.
For each $j$, select $i(j)$ such that $\Pi_{i(j)}u_j=u_j$.
Considering the above analysis of $\lambda_{i(j)}(t)=\sigma_j(t)^2$, it follows that $\lim_{t\to\infty}W(t)=N^{1/4}\sum_{\sigma_i(0)>0}u_iv_i^\top$.
Next, note that $\operatorname{ker}(W_0^\top)\supseteq\operatorname{span}\{1_C\}$ since $(H_0,W_0,b_0)\in S$.
Thus, $\operatorname{rank}W_0\geq C-1$ implies $\operatorname{ker}(W_0^\top)=\operatorname{span}\{1_C\}$, and so $\Pi=\frac{1}{C}1_C1_C^\top$.
As such, $\lim_{t\to\infty}G(t)$ satisfies \eqref{eq.snc1}, while \eqref{eq.snc2} and \eqref{eq.snc3} follow from taking limits of $H(t)$ and $b(t)$, respectively.
\end{proof}

At this point, there appears to be an ``unreasonable effectiveness'' of gradient descent in which locally minimizing empirical risk implicitly encourages strong neural collapse.
We conclude with the following result, which explains why this behavior is actually reasonable:

\begin{lemma}
For every $(H,W,b)\in S$, it holds that
\begin{equation}
\label{eq.empirical risk vs snc}
R_e(H,W,b)
=\frac{1}{2}\|WW^\top-\sqrt{N}(I-\frac{1}{C}1_C1_C^\top)\|_F^2+\frac{CN}{2}\|b-\frac{1}{C}1_C\|_2^2.
\end{equation}
In particular, the global minimizers of
\[
\text{minimize}
\qquad
R_e(H,W,b)
\qquad
\text{subject to}
\qquad
(H,W,b)\in S
\]
are precisely the $(H,W,b)\in\mathbb{R}^{p\times CN}\oplus\mathbb{R}^{C\times p}\oplus\mathbb{R}^C$ which exhibit strong neural collapse.
\end{lemma}

\begin{proof}
Suppose $(H,W,b)\in S$.
The implied form of $H$ then gives
\begin{align*}
WH+b1_{CN}^\top-I_C\otimes 1_N^\top
&=W(\frac{1}{\sqrt{N}}(W\otimes 1_N)^\top)+b1_{CN}^\top-(I_C-\frac{1}{C}1_C1_C^\top+\frac{1}{C}1_C1_C^\top)\otimes 1_N^\top\\
&=(\frac{1}{\sqrt{N}}WW^\top-(I_C-\frac{1}{C}1_C1_C^\top))\otimes 1_N^\top + (b-\frac{1}{C}1_C)1_{CN}^\top.
\end{align*}
Since $1_C^\top W=0$ and $b\in\operatorname{span}\{1_C\}$, it follows that the two terms on the right-hand side are orthogonal to each other.
The Pythagorean theorem then gives
\begin{align*}
R_e(H,W,b)
&=\frac{1}{2}\|(\frac{1}{\sqrt{N}}WW^\top-(I_C-\frac{1}{C}1_C1_C^\top))\otimes 1_N^\top\|_F^2 + \frac{1}{2}\|(b-\frac{1}{C}1_C)1_{CN}^\top\|_F^2\\
&=\frac{1}{2}\|WW^\top-\sqrt{N}(I-\frac{1}{C}1_C1_C^\top)\|_F^2+\frac{CN}{2}\|b-\frac{1}{C}1_C\|_2^2.
\end{align*}
For the second part of the result, first observe that the constraint $(H,W,b)\in S$ implies \eqref{eq.snc2}.
Subject to this constraint, \eqref{eq.empirical risk vs snc} gives that equality in $R_e(H,W,b)\geq0$ holds precisely when \eqref{eq.snc1} and \eqref{eq.snc3} both hold.
In addition, $(H,W,b)$ exhibits strong neural collapse only if $(H,W,b)\in S$.
The result follows.
\end{proof}

\section{Discussion}

In this paper, we introduced the \textit{unconstrained features model} that captures the neural collapse phenomena observed by Papyan, Han and Donoho~\cite{PapyanHD}, and then we identified an invariant subspace of the gradient descent dynamical system that encourages neural collapse.
While we have explained much about the emergence of neural collapse in the unconstrained features model, there are several directions for future investigation.
First, it would be nice to fully characterize the dynamics of gradient flow in the unconstrained features model.
How does the distance from $S$ behave over the full gradient flow trajectory?
To what extent does strong neural collapse hold for the limit point of this trajectory when the initialization is at most $\epsilon>0$ away from the origin?
What behaviors emerge from other popular local optimization methods such as Adaptive Moment Estimation~\cite{KingmaB:14}?
What alternatives to the unconstrained features model are amenable to theoretical treatment?

\section*{Acknowledgments}
DGM thanks Arje Nachman and Soledad Villar for (independently) bringing the article~\cite{PapyanHD} to his attention.
DGM was partially supported by AFOSR FA9550-18-1-0107 and NSF DMS 1829955. 
HP was partially supported by an AMS-Simons Travel Grant.


\begin{thebibliography}{WW}

\bibitem{AroraCH:18}
S.\ Arora, N.\ Cohen, E.\ Hazan,
On the optimization of deep networks:\ Implicit acceleration by overparameterization,
ICML 2018, 372--389.

\bibitem{BandeiraFMW:13}
A.\ S.\ Bandeira, M.\ Fickus, D.\ G.\ Mixon, P.\ Wong,
The road to deterministic matrices with the restricted isometry property, 
J.\ Fourier Anal.\ Appl.\ 19 (2013) 1123--1149.

\bibitem{ChizatB:18}
L.\ Chizat, F.\ Bach,
On the global convergence of gradient descent for over-parameterized models using optimal transport,
NeurIPS 2018, 3036--3046.

\bibitem{DuHL:18}
S.\ S.\ Du, W.\ Hu, J.\ D.\ Lee,
Algorithmic regularization in learning deep homogeneous models:\ Layers are automatically balanced,
NeurIPS 2018, 384--395.

\bibitem{DuZPS:18}
S.\ S.\ Du, X.\ Zhai, B.\ Poczos, A.\ Singh,
Gradient descent provably optimizes over-parameterized neural networks,
ICLR 2018.

\bibitem{FickusM:15}
M.\ Fickus, D.\ G.\ Mixon, 
Tables of the existence of equiangular tight frames,
arXiv:1504.00253

\bibitem{KingmaB:14}
D.\ P.\ Kingma, J.\ Ba,
Adam:\ A method for stochastic optimization,
arXiv:1412.6980

\bibitem{KrizhevskySH:12}
A.\ Krizhevsky, I.\ Sutskever, G.\ Hinton,
ImageNet Classification with Deep Convolutional Neural Networks, 
NIPS 2012, 1097--1105.

\bibitem{MixonQKM:13}
D.\ G.\ Mixon, C.\ J.\ Quinn, N.\ Kiyavash, M.\ Fickus,
Fingerprinting with equiangular tight frames,
IEEE Trans.\ Inf.\ Theory 59 (2013) 1855--1865.

\bibitem{UFLDL:online}
Multi-Layer Neural Network,
UFLDL Tutorial,
\url{http://ufldl.stanford.edu/tutorial/supervised/MultiLayerNeuralNetworks/}

\bibitem{PapyanHD}
V.\ Papyan, X.\ Y.\ Han, D.\ L.\ Donoho,
Prevalence of neural collapse during the terminal phase of deep learning training,
Proc.\ Natl.\ Acad.\ Sci.\ U.S.A.\ 117 (2020) 24652--24663.

\bibitem{Rankin:55}
R.\ A.\ Rankin,
The closest packing of spherical caps in $n$ dimensions,
In:\ Proceedings of the Glasgow Mathematical Association, vol.\ 2, Cambridge University Press, 1955, pp.\ 139--144.

\bibitem{RenesBSC:04}
J.\ M.\ Renes, R.\ Blume-Kohout, A.\ J.\ Scott, C.\ M.\ Caves, 
Symmetric informationally complete quantum measurements,
J.\ Math.\ Phys.\ 45 (2004) 2171--2180.

\bibitem{SongMN:18}
M.\ Song, A.\ Montanari, P.\ Nguyen,
A mean field view of the landscape of two-layers neural networks." Proc.\ Natl.\ Acad.\ Sci.\ U.S.A.\ 115 (2018) E7665--E7671.

\bibitem{StrohmerH:03}
T.\ Strohmer, R.\ W.\ Heath,
Grassmannian frames with applications to coding and communication,
Appl.\ Comput.\ Harmon.\ Anal.\ 14 (2003) 257--275.

\end{thebibliography}
\end{document}